\documentclass{article}
\usepackage{amsthm}
\newtheorem{theorem}{Theorem}[section]

\newtheorem{prop}{Proposition}[section]
\newtheorem{definition}{Definition}[section]
\usepackage{thmtools, thm-restate}
\usepackage{yhmath}
\usepackage{subfigure}
\newcommand{\rvw}{\mathbf{w}}
\usepackage{lipsum}
\usepackage{algorithm}
\usepackage{algpseudocode}
\usepackage{amsmath}
\usepackage{tikz}
\usepackage{enumitem}
\usepackage{algorithm}
\usepackage{algpseudocode}
\usepackage{xcolor}
\usepackage{booktabs}
\usepackage{mathtools}
\usepackage{hyperref}
\usepackage{nicefrac}
\hypersetup{
    colorlinks=true, 
    linktoc=all,
    urlcolor= {darkblue},
    linkcolor={black},
    citecolor={darkblue}
    }

\usepackage{thmtools, thm-restate}
\usepackage{yhmath}
\definecolor{darkblue}{rgb}{0.0, 0.0, 0.55}
 \definecolor{nicegreen}{rgb}{0.0, 0.5, 0.0}
   \definecolor{cornellred}{rgb}{0.7, 0.11, 0.11}
\usepackage[most]{tcolorbox}

\usepackage[nonatbib, final]{neurips_2023}
\usepackage[square,numbers]{natbib}
\usepackage[utf8]{inputenc} 
\usepackage[T1]{fontenc}    
\usepackage{hyperref} 
\usepackage{url}            
\usepackage{booktabs}       
\usepackage{amsfonts}       
\usepackage{nicefrac}       
\usepackage{microtype}      
\usepackage{xcolor}         

\title{Estimating optimal PAC-Bayes bounds with Hamiltonian Monte Carlo}

\author{%
  Szilvia Ujv{\'a}ry\thanks{Corresponding author} \\
  University of Cambridge\\
  \texttt{sru23@cam.ac.uk} \\
  \And
  Gergely Flamich \\
  University of Cambridge \\
  \texttt{gf332@cam.ac.uk} \\
  \newline
  \newline
  \And 
  Vincent Fortuin\\
    Helmholtz AI, TU Munich\\
  \texttt{vincent.fortuin@helmholtz-munich.de}\\
  \And
  Jos{\'e} Miguel Hern{\'a}ndez Lobato\\
  University of Cambridge \\
  \texttt{jmh233@cam.ac.uk} \\
}

\begin{document}

\maketitle

\begin{abstract}
An important yet underexplored question in the PAC-Bayes literature is how much tightness we lose by restricting the posterior family to factorized Gaussian distributions when optimizing a PAC-Bayes bound. 
We investigate this issue by estimating data-independent PAC-Bayes bounds using the optimal posteriors, comparing them to bounds obtained using MFVI. 
Concretely, we 
(1) sample from the optimal Gibbs posterior using Hamiltonian Monte Carlo, 
(2) estimate its KL divergence from the prior with thermodynamic integration, and 
(3) propose three methods to obtain high-probability bounds under different assumptions.
Our experiments on the MNIST dataset reveal significant tightness gaps, as much as 5-6\% in some cases.
\end{abstract}
\section{Introduction}
PAC-Bayes is a tool to give high-probability generalization bounds for (generalized\footnote{generalized Bayesian learning extends Bayesian learning by allowing more general losses and priors \citep{benjamin_primer}.}) Bayesian learning algorithms.
The main goal is to produce empirical bounds, also called risk certificates (RC), that are nonvacuous even in complex settings such as neural networks (NN). 
Current methods for NNs train a Bayesian Neural Network (BNN) use mean-field variational inference (MFVI) with an objective derived from a PAC-Bayes bound \citep{Karolina17, tighter} to obtain nonvacuous bounds.
However, the resulting bounds are only tight for relatively simple datasets (e.g. MNIST, CIFAR-10) and require data-dependent priors \citep{Karolina17, tighter, gordonwilson}. 
MFVI has been widely promoted in the PAC-Bayes literature as an efficient and accurate alternative to MCMC methods \citep{alquierVI, benjamin_primer}. In the Bayesian learning community, there is more controversy about the expressivity of MFVI \citep{foong, yarin}. Figure \ref{toy_figure} illustrates the trade-offs between MFVI and MCMC methods on the 2 dimensional Rosenbrock function \citep{rosenbrock}.
\par
We contribute to this debate by empirically estimating how tight PAC-Bayes bounds can be, compared to those obtained by MFVI. 
We consider data-independent bounds since these are always loose, hence the effects of using better weight distributions are more visible. 
For the bounds in our interest, the optimal (minimizing) probability measure on the weights has the form of a Gibbs 
distribution
\citep{alquier_userfriendly}. 
We use Hamiltonian Monte Carlo (HMC) \citep{hmc_original_paper, handbook_MCMC} to approximately sample from this distribution and then estimate the bound. 
We find that, especially for small datasets, our RCs can improve on MFVI significantly, such as an 5-6\% improvement over MFVI, resulting in a 10.8\% RC on Binary MNIST. This demonstrates the potential to significantly tighten PAC-Bayes bounds by considering more complicated distributions than factorized Gaussians.
\par 
Concretely, our contributions are as follows:
\setlist[enumerate]{leftmargin=2em}
\begin{enumerate}
    \item We estimate data-independent PAC-Bayes bounds via (1) approximately sampling from the optimal Gibbs posterior, (2) computing the KL divergence with thermodynamic integration and (3) ensuring a high-probability bound.
    \item We thereby demonstrate on versions of MNIST that MFVI bounds can be tightened significantly.
    \item We justify our approach by extensive diagnostic analysis on our HMC samples.
\end{enumerate}

\section{Background}
\begin{figure}[t]
    \centering
    \vspace{-0.15cm}
    \includegraphics[width=\textwidth]{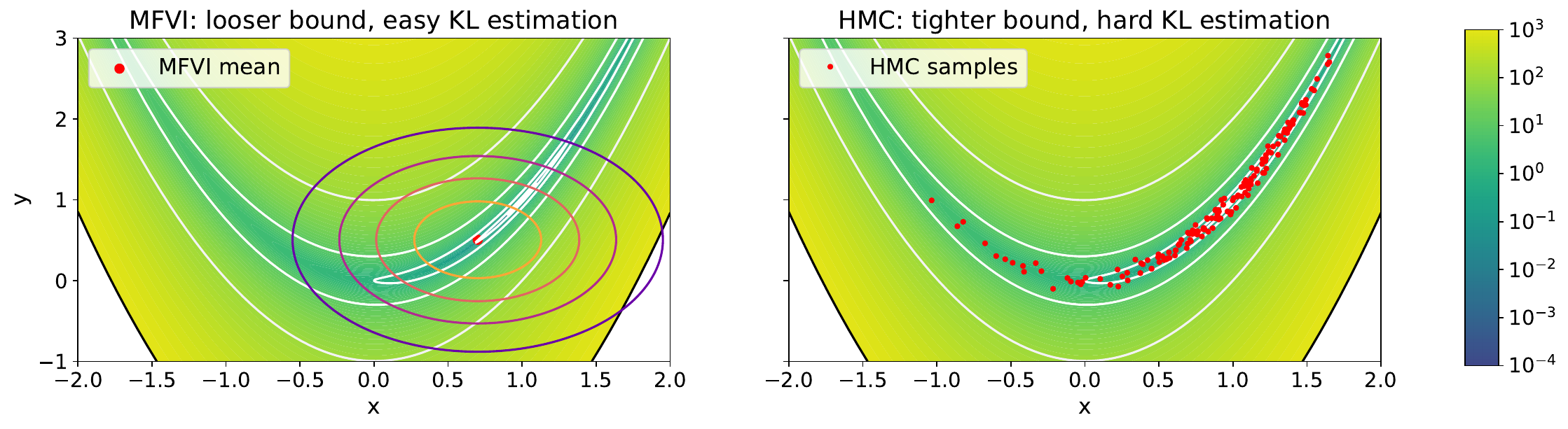}
    
    \caption{Illustration of the trade-offs of using MFVI and MCMC methods (e.g. HMC) in PAC-Bayes estimation. Our method solves the `hard KL estimation' problem. The 2-dimensional Rosenbrock function is used as loss function \citep{rosenbrock}. The prior is centered at (0, 1.5) and has unit variance. Both methods are initialized at (0, 3).}
    \label{toy_figure}
   \vspace{-0.1cm}
\end{figure}
We consider supervised classification tasks. Let $\mathcal{Z}=\mathcal{X} \times \mathcal{Y}$ be a measurable space, where $\mathcal{X} \subset \mathbb{R}^d$ is the space of data samples, and $\mathcal{Y}\subset \mathbb{R}$ is the space of $K$ labels. Let $\mathcal{M}_1(\mathcal{Z})$ denote the set of probability measures on $\mathcal{Z}$. Let $D$ be an unknown distribution $D \in \mathcal{M}_1(\mathcal{Z})$. A learning algorithm receives a set of $n$ i.i.d.\ samples $S=(z_1, ..., z_n)$, i.e. $S \sim D^n$. Further, we fix a weight space $\mathcal{W} \subseteq \mathbb{R}^p$ containing all possible weights. Each weight vector $\mathbf{w}$ maps to a \textit{predictor function} $h_{\mathbf{w}}: \mathcal{X} \to \mathcal{Y}$ that assigns a label $y \in \mathcal{Y}$ to any input $x \in \mathcal{X}$. 
We seek a predictor function that minimizes the \textit{risk} (expected loss) $L(\mathbf{w}):=\mathbb{E}_{z \sim D}\left[l(\mathbf{w}, z) \right]$,
where $l: \mathcal{W} \times \mathcal{Z} \to [0, \infty)$ is a measurable loss function. Let $l^{\mathrm{ce}}$ and $l^{\mathrm{0-1}}$ denote the cross-entropy and $0-1$ loss (error) functions, respectively.
The data-generating distribution being unknown, $L(\mathbf{w})$ is not observable, hence in practice, we compute the \textit{empirical risk}, which depends on our sample set $S$ and is defined as $\hat{L}_S(\mathbf{w}):= \frac{1}{n}\sum_{i=1}^n l(\mathbf{w}, z_i)$. PAC-Bayes bounds are stated for \textit{randomized predictors}. Given a data sample $x$, a randomized predictor makes a prediction at its label using a random sample of weights $\rvw \sim Q$, where $Q \in \mathcal{M}_1(\mathcal{W})$. We may identify the randomized predictor with its distribution $Q$.
The risk of $Q$ is defined as $L(Q):= \mathbb{E}_{\rvw \sim Q}[\mathbb{E}_{z \sim D}\left[l(\rvw, z) \right]]$.
The empirical risk of $Q$ is given by $\hat{L}_S(Q):= \frac{1}{n}\sum_{i=1}^n\mathbb{E}_{\rvw \sim Q}\left[l(\rvw, z_i) \right]]$.
We use notations $L^{\mathrm{ce}}(Q), \hat{L}_S^{\mathrm{ce}}(Q), L^{0-1}(Q)$ and $\hat{L}^{0-1}_S(Q)$ for our risk functionals.

\paragraph{PAC-Bayes bounds} We consider the following PAC-Bayes bounds. For proofs, see \cite{maurer}, \cite{lambda_bound}, \cite{alquier_userfriendly}.
\begin{theorem}
\label{seeger_theorem}
    Fix a loss function $l: \mathcal{W}\times \mathcal{Z} \to [0, 1]$ and fix arbitrary $\delta \in (0, 1)$. For any data-free distribution $P$ over $\mathcal{W}$, simultaneously for all distributions $Q$ over $\mathcal{W}$, with probability at least $1-\delta$, we have
    \begin{equation}
    \label{pac_kl}
    \textbf{kl bound:} \quad
        \mathrm{kl}(\hat{L}_S(Q)||L(Q))\leq \frac{\mathrm{KL}(Q||P)+\log(\frac{2\sqrt{n}}{\delta})}{n}.
    \end{equation}
Further, simultaneously for all $Q \in \mathcal{W}$ and $\lambda \in (0, 2)$, we have the below relaxation of the kl bound   
\vspace{-0.1cm}
\begin{equation}
\lambda \textbf{ bound:} \quad
        L(Q) \leq \frac{\hat{L}_S(Q)}{1-\frac{\lambda}{2}} + \frac{\mathrm{KL}(Q||P)+ \log(\frac{2\sqrt{n}}{\delta})}{n \lambda (1-\frac{\lambda}{2})}.
\label{lambda_bound_Def}
\end{equation}
\end{theorem}

\begin{theorem}
\label{gibbs_density}
    The optimal stochastic predictor corresponding to the $\lambda$ bound with prior $P$ that has density $p(\rvw)$ is a Gibbs measure $Q^*_{\lambda}$ with density $q^*_{\lambda}(\rvw|z)=e^{-n\lambda \hat{L}_S(\mathbf{w})}p(\rvw) / \mathbb{E}_{\mathbf{w} \sim P}\left[e^{-n\lambda \hat{L}_S(\mathbf{w})}\right]$.

\end{theorem}

For simplicity, we use the $\lambda$ bound with $\lambda=1$ fixed. Even this simple setting improves over MFVI risk certificates (RCs) .

\paragraph{Empirical estimation of PAC-Bayes bounds} RCs are most often computed using MFVI on a BNN with an objective function derived from a PAC-Bayes bound. For posterior sampling, a bounded version of the cross-entropy loss $\widetilde{l^{\mathrm{ce}}}$ (defined in Supplementary Material \ref{trans_loss}) is used. We use notation $\widetilde{L}_S^{\mathrm{ce}}(Q)$ for when we use the compute empirical risks with $\widetilde{l^{\mathrm{ce}}}$. The final RCs can be computed in terms of the $0-1$ loss too. 
The KL divergence of the final approximate posterior with respect to the prior is available analytically. The risk $\hat{L}_S(Q)$ is approximated as a Monte Carlo average over samples $\mathbf{w} \sim Q$. The following concentration inequality is then used to 
ensure an upper-bound on $L(Q)$ with pre-specified probability $1-\delta$ \citep{caruana}.
\begin{theorem}
\label{caruana_theorem}
    Suppose $W_1, W_2, ..., W_m \sim Q$ are i.i.d., and $\hat{Q}_m=\sum_{j=1}^m \delta_{W_j}$ is their empirical distribution. Then for any $\delta' \in (0, 1)$ with probability $1-\delta'$, we have that 
    ${\hat{L}_S(Q) \leq \mathrm{kl} ^{-1}\left( \hat{L}_S(\hat{Q}_m), \frac{1}{m}\log(\frac{2}{\delta'}) \right).}$
\end{theorem}
\section{Method}

\begin{algorithm}[t]
\caption{PAC-Bayes estimation with HMC}\label{method_alg}
\begin{algorithmic}
\State Sample $\rvw \sim Q^*$ with HMC,
\State Estimate $\hat{L}_S^{\mathrm{0-1}}(Q^*)$ from $\rvw$ using a high-probability bound,\Comment{Suppl. M. \ref{high-prob_bounds_section}}
\For{$\beta$ in [0, 1] discretization}
    \State Sample $\rvw_\beta \sim Q^{*\beta} \propto e^{-\beta\widetilde{L}_S^{\mathrm{ce}}(\rvw)}p(\rvw)$ with HMC,
    \State Estimate $\widetilde{L}_S^{\mathrm{ce}}(Q^{*\beta})$ from $\rvw_\beta$ using a high-probability bound \Comment{Suppl. M. \ref{high-prob_bounds_section}}
\EndFor
\State Estimate $\log p(z)$ with thermodynamic integration, \Comment{Eqn. \ref{ti_logz_eq}}
\State Compute $\mathrm{KL}(Q^*||P)$, \Comment{Prop \ref{caruana_theorem}}
\State Compute bounds in Eqs. \ref{pac_kl}, \ref{lambda_bound_Def}  using $\hat{L}_S^{\mathrm{0-1}}(Q^*)$ and $\mathrm{KL}(Q^*||P)$ .

\end{algorithmic}
\end{algorithm}

\noindent
Our goal is to adapt the above pipeline to more general posterior approximations, specifically HMC samples with a Gibbs target density. Our method is summarized in Algorithm \ref{method_alg}. The key challenges are estimating $\mathrm{KL}(Q^*||P)$ and the fact that MCMC samples are not independent. 
\vspace{-0.15cm}
\paragraph{Approximate sampling from the Gibbs posterior}
As HMC requires gradients of the unnormalized target distribution, the $\widetilde{l^{\mathrm{ce}}}$ loss is used to generate samples with the numerator of the Gibbs density formula in Theorem \ref{gibbs_density} as the target density.
\vspace{-0.15cm}
\paragraph{Estimating the KL divergence}
Having obtained approximate samples from $Q^*$ as described above, we now estimate $\mathrm{KL}(Q^*||P)$ by reducing the problem to (generalised) marginal likelihood estimation via the following fact.
\begin{prop}
\label{kl_eq}
Fix a prior measure $P$ with density $p(\rvw)$ and define a (possibly generalized) likelihood $p(z | \rvw)$. For any distribution $Q$ that is dependent on the data $z$, with density $q(\rvw|z)$ the following holds
    \begin{equation}
     \mathrm{KL}(q(\rvw|z)||p(\rvw)) = \mathbb{E}_q \left[\log (p(z|\rvw)) \right]  - \log(Z)  + \mathrm{KL}(q(\rvw|z) || p(\rvw|z)).
\end{equation}
where $Z=p(z)$, the (generalised) marginal likelihood.
\end{prop}
For $Q=Q^*$ the true posterior, the last term is $0$ and the term of most interest becomes the generalised log marginal likelihood $\log(Z)$.
We estimate this term using thermodynamic integration \citep{TI} as
\begin{equation}
    -\log (Z) = \int_{0}^1 \mathbb{E}_{\rvw \sim \pi_{\beta}} \left[n   \hat{L}_S(\mathbf{w}) \right] \mathrm{d}\beta, \text{ where } \pi_{\beta} \propto e^{-\beta   \hat{L}_S(\mathbf{w})}p(\rvw).
    \label{ti_logz_eq}
\end{equation}
Notice that $\pi_\beta$ is itself a Gibbs distribution, hence computing the thermodynamic integral reduces to sampling from Gibbs distributions for values of $\beta \in [0, 1]$. The details of this derivation are given in Supplementary Material \ref{using_ti}. When computing the integral, to control the tail of our bound, it is necessary to ensure that we overestimate its true value with high probability. We use the trapezium rule, which we show in Supplementary Material \ref{ti_results} to upper-bound the integral.
 
\paragraph{High-probability upper bounds}
The last step is to control the overall probability of our bound using concentration inequalities on the tails of all of our estimators. Note that Theorem \ref{caruana_theorem} used with MFVI is valid only for independent samples. We propose the following three options for the general setting. Each has different assumptions and their comparison serves as a sanity check on the magnitude of our estimates. Derivations are given in Supplementary Material \ref{high-prob_bounds_section}.
\setlist[itemize]{leftmargin=2em}
\begin{itemize}
    \item We use Theorem \ref{caruana_theorem} after thinning our samples to reduce autocorrelations.\\
    \textit{This is the simplest option, but we can only guarantee uncorrelatedness.}
    \item An asymptotic, probability $(1 - \alpha/2)$ confidence interval of the form ${\left[0, c \right)}$, where
    \begin{align}
    c = \frac{(1 + \epsilon)\hat{\sigma}_m}{\sqrt{2\alpha m}} + \frac{1}{m}\sum_{i=1}^m \hat{L}_S(\rvw_i).
    \end{align}
    This bound requires assumptions on the order of the bias and variance of our estimator $\frac{1}{m}\sum_{i=1}^m \hat{L}_S(\rvw_i)$, where $\rvw_i \sim Q$, and a consistent estimator of its asymptotic variance \cite{clt_stuff}. \\
    \textit{This gives the tightest RCs, and has reasonable assumptions but they are hard to check.}
    
    \item Using a sanity-check bound $\mathrm{KL}(Q||P) \leq n\mathbb{E}_G[\widetilde{L}_S^{\mathrm{ce}}(\rvw)] + \mathrm{KL}(G||P)$. In Prop. \ref{sillyprop}, we prove this result for any $G$ satisfying $\mathrm{KL}(Q||Q^*) \leq \mathrm{KL}(G||Q^*) + n\mathbb{E}_{Q}\left[\widetilde{L}_S^{\mathrm{ce}}(\rvw) \right] $, where $Q^*$ is the Gibbs posterior. Taking $G$ to be the Gaussian from MFVI, the assumption only requires that $Q$ is not much further from the true posterior and $G$. If $\mathbb{E}_{Q}\left[\widetilde{L}_S^{\mathrm{0-1}}(\rvw) \right] < \mathbb{E}_{G}\left[\widetilde{L}_S^{\mathrm{0-1}}(\rvw) \right]$, this bound can be tighter than MFVI. \\
    \textit{This has the mildest assumptions and hence is easiest to estimate, but is the most conservative.}

\end{itemize}

\section{Experiments and results}

We apply our method to three versions of the MNIST train dataset: Binary MNIST, where labels are mapped $y \mapsto y \text{ (mod } 2)$, MNIST reduced to $14 \times 14$ pixels and full MNIST. We also compare results when using only half of the 60,000 datapoints. We use a data-independent factorized Gaussian prior and small MLPs. For each configuration, we generate 4 independent HMC chains of 5000 samples each. We compare results to estimating the same RCs with MFVI. All RCs are valid with probability at least 0.95. Further details are given in Supplementary Material \ref{hyperparam}.
\paragraph{HMC diagnostics}
\begin{figure}[t]
    \centering
 
\begin{subfigure}
         \centering
    \includegraphics[width=0.95\textwidth]{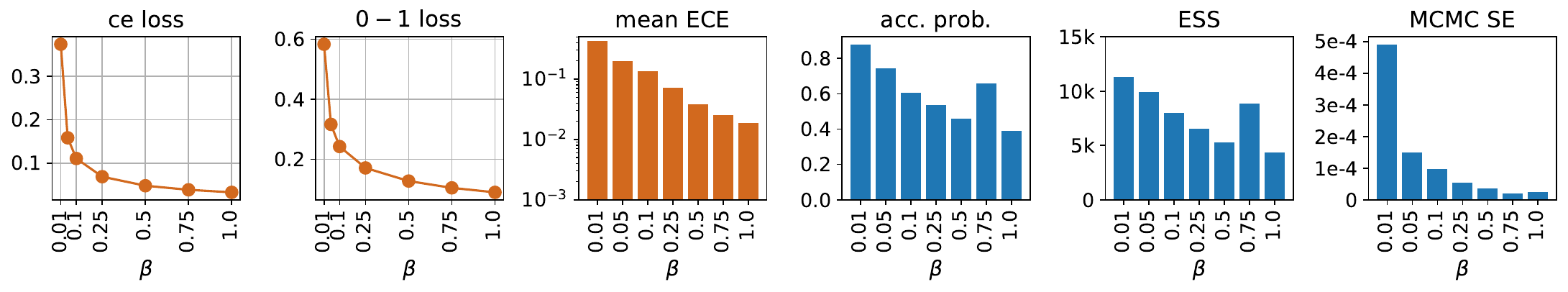}

\end{subfigure}

\begin{subfigure}
         \centering
    \includegraphics[width=0.95\textwidth]{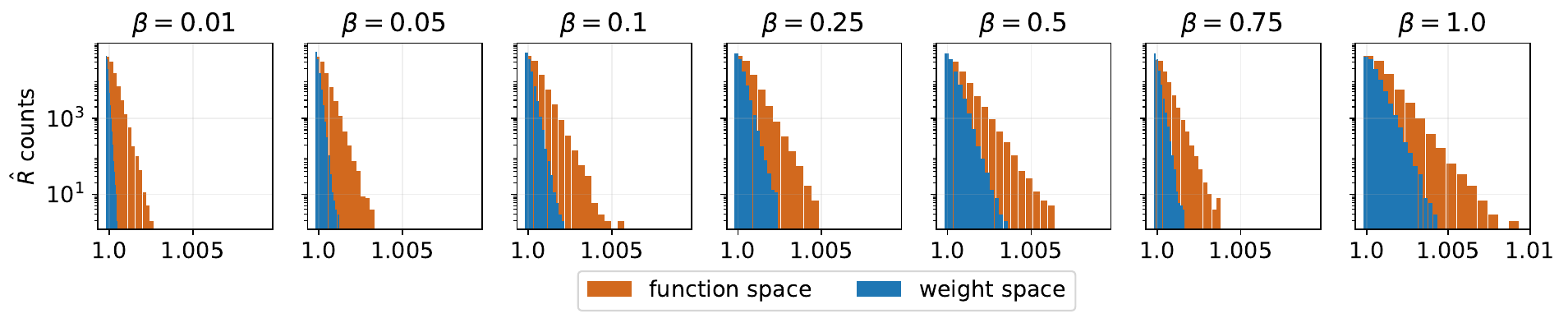}
\end{subfigure}
\caption{HMC diagnostics for MNIST - half, over 4 chains of length 5000, for each value of $\beta$.}
\label{half_diag_figure_}
\end{figure}
Figure \ref{half_diag_figure_} shows HMC diagnostics for the MNIST-half dataset. The rest of the diagnostic results are supplied in Supplementary Material \ref{further_diagnostics_Section}.
We report expected calibration error, HMC acceptance probability, effective sample size, MCMC standard error and $\hat{R}$ statistics along with cross-entropy and $0-1$ losses for all datasets and $\beta$ values. Unsurprisingly, performance increases as $\beta$ increases, since our samples get more similar to the true posterior. The $\hat{R}$ statistic is very close to 1.0 both in weight and function space, indicating approximate convergence in our chains. 



\newpage
\paragraph{Bound calculation results}
Table \ref{rc_table} shows the results of our RC estimation, compared to MFVI. Our method always improves $0-1$ RCs, most significantly for Binary MNIST, where the gap can reach 10\% in some configurations. The gap increases when more data is used. More analysis is given in Supplementary Material \ref{further_results_Section}.

\begin{table}[t]
\caption{Training and test metrics and RC estimates using our three bounds. \textbf{Bold} numbers indicate the tighter certificate out of the Gibbs and MFVI ones for the same dataset. }
\resizebox{\textwidth}{!}{
\begin{tabular}{@{}llllllllllll@{}}
\toprule
\multicolumn{3}{l}{Setup}            & \multicolumn{3}{l}{Train/test stats} & \multicolumn{3}{l}{0-1 RC with kl bound}           & \multicolumn{3}{l}{0-1 RC with $\lambda$ bound}    \\  \midrule
Method   & Dataset             & Model & Train 0-1    & Test 0-1    & KL/n     & kl inverse      & asympt          & naive           & kl inverse      & asympt          & naive           \\ \midrule
MFVI     & Binary Half         & 1L    & 0.0924       & 0.0982      & 0.0106   & 0.1600          & 0.1426          & 0.1600          & 0.2385          & 0.2090          & 0.2385          \\
Gibbs p. & Binary Half         & 1L    & 0.0562       & 0.0561      & 0.0205   & \textbf{0.1342} & \textbf{0.1065} & \textbf{0.1417} & \textbf{0.1820} & \textbf{0.1428} & \textbf{0.1877} \\
MFVI     & Binary              & 1L    & 0.0960       & 0.0928      & 0.0105   & 0.1640          & 0.1452          & 0.1640          & 0.2452          & 0.2136          & 0.2452          \\
Gibbs p. & Binary              & 1L    & 0.0404       & 0.0415      & 0.0195   & \textbf{0.1080} & \textbf{0.0702} & \textbf{0.1184} & \textbf{0.1435} & \textbf{0.0969} & \textbf{0.1566} \\ \midrule
MFVI     & $14 \times 14$ Half & 2L    & 0.1348       & 0.1319      & 0.0148   & 0.2449          & 0.2059          & 0.2449          & 0.3595          & 0.3070          & 0.3595          \\
Gibbs p. & $14 \times 14$ Half & 2L    & 0.0954       & 0.1010      & 0.0477   & \textbf{0.1888} & \textbf{0.1805} & \textbf{0.2324} & \textbf{0.3273} & \textbf{0.2481} & \textbf{0.3180} \\
MFVI     & $14 \times 14$      & 2L    & 0.1389       & 0.1313      & 0.0140   & 0.2379          & 0.1991          & 0.2379          & 0.3595          & 0.2930          & 0.3595          \\
Gibbs p. & $14 \times 14$      & 2L    & 0.0695       & 0.0723      & 0.0381   & \textbf{0.1855} & \textbf{0.1335} & \textbf{0.1920} & \textbf{0.2507} & \textbf{0.1810} & \textbf{0.2600} \\ \midrule
MFVI     & MNIST Half          & 2L    & 0.1256       & 0.1264      & 0.0199   & 0.2302          & 0.2025          & 0.2302          & 0.3387          & 0.2911          & 0.3387          \\
Gibbs p. & MNIST Half          & 2L    & 0.0898       & 0.0970      & 0.0428   & \textbf{0.2248} & \textbf{0.1740} & \textbf{0.2247} & \textbf{0.3091} & \textbf{0.2377} & \textbf{0.3068} \\
MFVI     & MNIST               & 2L    & 0.1236       & 0.1200      & 0.0196   & 0.2070          & 0.1987          & 0.2070          & 0.2977          & 0.1714          & 0.2977          \\
Gibbs p. & MNIST               & 2L    & 0.0653       & 0.0691      & 0.0334   & \textbf{0.1759} & \textbf{0.1269} & \textbf{0.1880} & \textbf{0.2381} & \textbf{0.2822} & \textbf{0.2553} \\ \bottomrule
\end{tabular}}
\label{rc_table}
\end{table}

\section{Discussion}
\paragraph{Related work} 
Our work opposes the view that MFVI does not considerably loosen PAC-Bayes bounds. The theoretical paper \citet{alquierVI} argues for this view by showing that Gibbs posteriors concentrate around the global minimum of the loss at the same rate as their best Gaussian approximations. However, their result is proven for binary linear classification only, and the best possible Gaussian approximation may not be attainable in practice. Many other works criticize MFVI especially for small neural nets, including \citet{foong} and \citet{yarin} in Bayesian learning. In the PAC-Bayes setting, \citet{pitas} experimentally demonstrates the limitations of MFVI in PAC-Bayes and improves RCs using a simplified KFAC Laplace approximation. To our knowledge, we are the first to tightly estimate the optimal value of PAC-Bayes bounds (i.e. in the Gibbs posterior), with the closest work being \citet{Karolina18}, who attempt this task in the context of data-dependent priors. They find that their bound is very loose in practice, which may be due to not using tempering in the KL estimation. The work \citet{rich_paper}  evaluates the tightness limits of the whole PAC-Bayes framework (rather than specific bounds) and restricts to very small datasets (30-60 datapoints). Our work builds on \citet{izmailov} and \citet{wenzel2020good} who scaled HMC up to neural networks. 

\paragraph{Limitations}
The chief limitation of our method is that PAC-Bayes bound estimates require assumptions that are impossible to completely verify for MCMC samples. What is possible are empirical estimates accompanied by robust diagnostics. Hence we make claims under sets of assumptions of varying strength. Even our most conservative estimates demonstrate the inaccuracy of MFVI.

\paragraph{Conclusion} 
We have demonstrated that using MFVI in PAC-Bayes estimation leads to data-independent bounds that are much looser than necessary. We provided estimates of optimal bounds on MNIST versions, using HMC and thermodynamic integration. Our estimates have plausible magnitude and are supported by extensive diagnostic analysis. The improvement over MFVI is largest for small models and adding more data tightens our estimates more than their MFVI approximations. Our results demonstrate the need for better posterior approximations for tight bounds.

\begin{ack}
The authors would like to thank Ferenc Huszár for helpful conversations and feedback on the manuscript. SU was supported by an MPhil scholarship from Emmanuel College, Cambridge, and the Cambridge Trust. GF acknowledges funding from DeepMind. VF was supported by a Branco Weiss Fellowship.

\end{ack}
\bibliographystyle{plainnat}
\newpage
\bibliography{refs}

\newpage
\section{Supplementary Material}
\subsection{Overview of contents}
The supplementary material is structured as follows. Section \ref{prop31} proves the simple result needed to deduce KL estimation to marginal likelihood estimation. Sections \ref{intro_ti}, \ref{using_ti} and \ref{ti_results} are centered around thermodynamic integration. We introduce the method, apply it in our setting and justify why the trapezium rule gives an upper bound on our thermodynamic integral. Section \ref{high-prob_bounds_section} details our high-probability bounds used to control the probability of our PAC-Bayes estimates, including their assumptions. Then, Section \ref{implementation_section} details technical aspects of the implementation and hyperparameters used. Finally, Sections \ref{further_diagnostics_Section} and \ref{further_results_Section} supply further diagnostic and RC estimation results and analysis.

\paragraph{A note on notation} Throughout the supplementary material, for the sake of generality, we state our results for general $\lambda$ values. However, the reader should keep in mind that we fix $\lambda=1$ in our experiments.

\subsection{Details on the method}
In this section, we provide further details, derivations, and proofs necessary for our method. 

\subsubsection{Proof of Proposition 3.1}
\label{prop31}
\begin{proof}
we can write the ELBO in two ways:
\begin{equation}
    \mathbb{E}_q \left[\log  \frac{p(z, \rvw)}{q(\rvw|z)} \right]=\mathbb{E}_q \left[\log   \frac{p(\rvw)p(z|\rvw)}{q(\rvw|z)}  \right]=\mathbb{E}_q \left[\log \frac{p(z)p(\rvw|z)}{q(\rvw|z)}  \right]
\end{equation}
We can expand both terms as:
\begin{equation}
    \mathbb{E}_q \left[\log p(z|\rvw) \right] + \mathbb{E}_q \left[\log \frac{p(\rvw)}{q(\rvw|z)} \right] = \mathbb{E}_q \left[\log p(z) \right] + \mathbb{E}_q \left[\log \frac{p(\rvw|z)}{q(\rvw|z)} \right],
\end{equation}
which we can rewrite as 
\begin{equation}
    \mathbb{E}_q \left[\log p(z|\rvw) \right] - \mathbb{E}_q \left[\log \frac{q(\rvw|z)}{p(\rvw)} \right] = \mathbb{E}_q \left[\log p(z) \right] - \mathbb{E}_q \left[\log \frac{q(\rvw|z)}{p(\rvw|z)} \right].
\end{equation}
We can now rewrite this as KL divergences and notice that $p(z)$ is independent of $\rvw$, hence
\begin{equation}
    \mathbb{E}_q \left[\log p(z|\rvw) \right] - \mathrm{KL}[q(\rvw|z)||p(\rvw)] = \log p(z) - \mathrm{KL}[q(\rvw|z) || p(\rvw|z)].
\end{equation}
We can reorder this as: 
\begin{equation}
     \mathrm{KL}[q(\rvw|z)||p(\rvw)] = \mathbb{E}_q \left[\log p(z|\rvw) \right]  - \log p(z)  + \mathrm{KL}[ q(\rvw|z) || p(\rvw|z)].
\end{equation}
\end{proof}
\subsubsection{Introduction to Thermodynamic Integration}
\label{intro_ti}
Thermodynamic integration (TI) is a physics-inspired method that allows us to approximate intractable normalizing constants of high dimensional distributions \citep{TI}. The main insight is to transform the problem into estimating the difference of two log normalizing constants. Since we are required to estimate $\log(Z)$, this framework suits our purposes.

Consider two probability measures $\Pi_1, \Pi_2 \in \mathcal{M}_1(\mathcal{W})$ with corresponding densities $\pi_1(\rvw), \pi_2(\rvw)$ and their unnormalized versions

\begin{equation}
\pi_{i}(\mathbf{w})=\frac{\tilde{\pi}_{i}(\mathbf{w})}{W_{i}}, \quad \quad W_{i}=\int \tilde{\pi}(\mathbf{w}) \mathrm{d} \mathbf{w}, \quad i \in\{0,1\}. 
\end{equation}

\noindent To apply TI, we form a geometric path  between $\pi_{0}(\mathbf{w})$ and $\pi_{1}(\mathbf{w})$ via a scalar parameter $\beta \in[0,1]$:

\begin{equation}
    \pi_{\beta}(\mathbf{w})=\frac{\tilde{\pi}_{\beta}(\mathbf{w})}{W_{\beta}}=\frac{\tilde{\pi}_{1}(\mathbf{w})^{\beta} \tilde{\pi}_{0}(\mathbf{w})^{1-\beta}}{W_{\beta}}, \quad W_{\beta}=\int \tilde{\pi}_{\beta}(\mathbf{w}) \mathrm{d} \mathbf{w}, \quad \beta \in[0,1].
\end{equation}
The central identity of thermodynamic integration is as follows. The right-hand side of Equation \ref{ti_eq} is referred to as the \textit{thermodynamic integral}.

\begin{prop}
Define the \textit{potential} as $U_{\beta}(\mathbf{w})=\log \tilde{\pi}_{\beta}(\mathbf{w})$ and let $U_{\beta}^{\prime}(\mathbf{w}):=\frac{\partial}{\partial \beta} U_{\beta}(\mathbf{w})$. Then,
\begin{equation}
\label{ti_eq}
   \log \left(W_{1}\right)-\log \left(W_{0}\right)=\int_{0}^{1} \mathbb{E}_{\rvw \sim \pi_{\beta}}\left[U_{\beta}^{\prime}(\mathbf{w})\right] \mathrm{d}\beta.
\end{equation}
\end{prop}
\begin{proof}
Please refer to Appendix A in \cite{TI}.
\end{proof}
\subsubsection{Using Thermodynamic Integration}
\label{using_ti}
As seen in Section \ref{intro_ti}, thermodynamic integration requires defining two probability measures with unnormalized densities, $\tilde{\pi}_i(\mathbf{w})$ and calculates the difference of the log normalizing constants $\log (W_1) - \log (W_0).$ If $\tilde{\pi}_0$ is defined such that $W_0=1$, then TI calculates the log normalizing constant of $\tilde{\pi}_1$. We thus define
\begin{equation}
    \tilde{\pi}_0(\mathbf{w}) := p(\mathbf{w}), \quad
    \tilde{\pi}_1(\mathbf{w}) := p(z, \mathbf{w}),
\end{equation}
thus $W_0=1$ and $W_1=\int p(z, \mathbf{w}) d\rvw = Z$. In TI, we then define the geometric path for $\beta \in [0, 1]$:
\begin{equation}
    \tilde{\pi}_{\beta}(\mathbf{w}):=p(z, \mathbf{w})^{\beta}p(\mathbf{w})^{1-\beta}.
\end{equation}
Then we have that $U_{\beta}(\mathbf{w})=\log \tilde{\pi}_{\beta}(\mathbf{w}) = \beta \log p(z, \mathbf{w}) + (1-\beta)\log p(\mathbf{w})$ and thus
\begin{equation}
    \frac{\partial}{\partial\beta}U_{\beta}(\mathbf{w}) = \log{\frac{p(z, \mathbf{w})}{p(\mathbf{w})}} = \log{\frac{e^{-n\lambda \widetilde{L}^{\mathrm{ce}}_S(\mathbf{w})} p(\mathbf{w})}{p(\mathbf{w})}}=-n\lambda \widetilde{L}^{\mathrm{ce}}_S(\mathbf{w}).
\end{equation}
Then, the thermodynamic integration formula (Equation \ref{ti_eq} yields the following form for the log normalizing constant  
\begin{equation}
    \log(Z) = \int_{0}^1 \mathbb{E}_{\rvw \sim \pi_{\beta}} \left[-n \lambda \widetilde{L}^{\mathrm{ce}}_S(\mathbf{w}) \right] \mathrm{d}\beta.
\end{equation}
We will estimate $\mathbb{E}_{\rvw \sim \pi_{\beta}} \left[-n \lambda \widetilde{L}^{\mathrm{ce}}_S(\mathbf{w}) \right]$ with a Monte Carlo average, using samples $\rvw \sim \pi_{\beta}$. To sample from each $\pi_{\beta}$, we can use HMC again to draw from the log joint
\begin{flalign}
    \log \tilde{\pi}_{\beta}(\mathbf{w})&=\beta \log p(z, \mathbf{w}) + (1-\beta) \log p(\mathbf{w}) \\
    &= - \beta \lambda n \widetilde{L}^{\mathrm{ce}}_S(\mathbf{w}) + \beta \log p(\mathbf{w}) + (1-\beta) \log p(\mathbf{w}) \\
    &= -\beta \lambda n \widetilde{L}^{\mathrm{ce}}_S(\mathbf{w}) + \log p(\mathbf{w}).
\end{flalign}


\subsubsection{The trapezium rule upper bounds the thermodynamic integral}
\label{ti_results}
In this section, we state our results for the \textit{negative} log normalizing constant, $-\log(Z)$ for convenience.
A naive strategy for calculating the integrand for $\log(Z)$ is to use a Monte Carlo average for the integral. 
\begin{flalign}
\label{upper_sums}
    -\log(Z) &= \int_{0}^1 \mathbb{E}_{\rvw \sim \pi_{\beta}} \left[n \lambda \widetilde{L}^{\mathrm{ce}}_S(\mathbf{w}) \right] \mathrm{d}\beta \\
    & \approx \frac{1}{B} \sum_{b=0}^{B-1} \left[\mathbb{E}_{\rvw \sim \pi_{\frac{b}{B}}} \left[n \lambda \widetilde{L}^{\mathrm{ce}}_S(\mathbf{w}) \right]\right]
\end{flalign}
Notice that the right-hand-side in Equation \ref{upper_sums} is a left Riemann sum (i.e. the integral is approximated at the left end of each subinterval of $[0, 1]$) on the function $\mathbf{w} \mapsto \mathbb{E}_{\rvw \sim \pi_{b/B}} \left[n \lambda \widetilde{L}^{\mathrm{ce}}_S(\mathbf{w}) \right]$, where each subinterval has length $\frac{1}{B}$.

The fact that this quantity upper bounds the KL divergence is established in \cite{TI} by showing that the integrand is a decreasing function of $\beta$. 
Let $g(\beta)=\mathbb{E}_{\rvw \sim \pi_{\beta}} \left[-U'_{\beta}(\rvw) \right]=\mathbb{E}_{\rvw \sim \pi_{\beta}} \left[n \lambda \widetilde{L}^{\mathrm{ce}}_S(\mathbf{w}) \right]$, our integrand. We notice that 
$U'_{\beta}(\rvw)=-n \lambda \widetilde{L}^{\mathrm{ce}}_S(\mathbf{w})$ is independent of $\beta$, hence we may abandon the subscript and use notation $U'(\rvw)=-n \lambda\widetilde{L}^{\mathrm{ce}}_S(\mathbf{w}))$. \cite{TI} show that
\begin{prop} $\frac{\partial g(\beta)}{\partial \beta} = - \mathrm{Var}_{\rvw \sim \pi_{\beta}}\left[U'(\rvw) \right]\leq 0.$
\label{paperprop}
\end{prop}
\begin{proof}
    Please refer to Appendix A in \cite{TI}. Note that they define $g(\beta)$ as the \textit{negative} of our $g(\beta)$.
\end{proof}
Hence $\mathbb{E}_{\rvw \sim \pi_{\beta}} \left[-U'(\rvw) \right]$ is a monotonically decreasing function, and this shows that calculating the left sums upper bounds our integral. In fact, we can show more. Below we show that $g(\beta)$ is convex, hence the trapezium rule can also be used to upper bound this integral. 

\begin{restatable}{prop}{convexityprop}
\label{convexityprop}
Let $g(\beta)=\mathbb{E}_{\rvw \sim \pi_{\beta}} \left[-U'(\rvw) \right]=\mathbb{E}_{\rvw \sim \pi_{\beta}} \left[n \lambda \widetilde{L}^{\mathrm{ce}}_S(\mathbf{w})\right]$.
Then we have that $\frac{\partial^2 g(\beta)}{\partial \beta^2} \geq 0$, hence $g(\beta)$ is convex.
\end{restatable}
\begin{proof}
We make use of the fact that $\frac{\partial }{\partial \beta}\mathbb{E}_{\rvw \sim \pi_{\beta}} \left[-U'(\rvw) \right] =-\text{Var}_{\rvw \sim \pi_{\beta}} \left[U'(\rvw) \right]$, proved in \cite{TI} (note that they define $g(\beta)$ to be the negative of our $g(\beta)$, hence the two statements differ by a minus sign). Plugging this in and expanding the variance, we obtain
\begin{flalign}
    \frac{\partial^2 g(\beta)}{\partial \beta^2} &= \frac{\partial^2 }{\partial \beta^2}\mathbb{E}_{\rvw \sim \pi_{\beta}} \left[-U'(\rvw) \right]\\
    &=-\frac{\partial}{\partial \beta} \left[ \mathbb{E}_{\rvw \sim \pi_{\beta}} \left[(U'(\rvw))^2 \right] - \mathbb{E}_{\rvw \sim \pi_{\beta}} \left[U'(\rvw) \right]^2\right]\\
    &=-\frac{\partial}{\partial \beta} \mathbb{E}_{\rvw \sim \pi_{\beta}} \left[(U'(\rvw))^2 \right] +2 \mathbb{E}_{\rvw \sim \pi_{\beta}} \left[U'(\rvw) \right]\frac{\partial}{\partial \beta}\mathbb{E}_{\rvw \sim \pi_{\beta}} \left[U'(\rvw) \right]\\
    &= A + B.
\end{flalign}
For $B$ we can use the above, i.e.
\begin{flalign}
    B
    &= 2\mathbb{E}_{\rvw \sim \pi_{\beta}} \left[U'(\rvw) \right]\text{Var}_{\rvw \sim \pi_{\beta}} \left[U'(\rvw) \right] \\
    &= 2\mathbb{E}_{\rvw \sim \pi_{\beta}} \left[U'(\rvw) \right] \mathbb{E}_{\rvw \sim \pi_{\beta}} \left[(U'(\rvw))^2 \right] -2\mathbb{E}_{\rvw \sim \pi_{\beta}} \left[U'(\rvw) \right]^3.
\end{flalign}
For $A$, note that $U'(\rvw)$ is \textit{independent of} $\beta$, hence we can directly plug into equation (30) in \cite{TI}, noting that we have an extra $U'(\rvw)$ term and we need the negative of their expression:
\begin{flalign}
    A &= \int (U'(\rvw))^2 \pi_{\beta}(\rvw)\mathbb{E}_{\rvw \sim \pi_{\beta}(\rvw)} \left[U'(\rvw) \right]  \mathrm{d}\rvw - \int (U'(\rvw))^3 \pi_{\beta}(\rvw)\mathrm{d}\rvw\\
    &=\mathbb{E}_{\rvw \sim \pi_{\beta}} \left[U'(\rvw) \right] \mathbb{E}_{\rvw \sim \pi_{\beta}} \left[(U'(\rvw))^2 \right] - \mathbb{E}_{\rvw \sim \pi_{\beta}} \left[(U'(\rvw))^3 \right].
\end{flalign}

Hence $A+B$ is
\begin{flalign}
    A+B = -2\mathbb{E}_{\rvw \sim \pi_{\beta}} \left[U'(\rvw) \right]^3 - \mathbb{E}_{\rvw \sim \pi_{\beta}} \left[(U'(\rvw))^3 \right] + 3\mathbb{E}_{\rvw \sim \pi_{\beta}} \left[U'(\rvw) \right] \mathbb{E}_{\rvw \sim \pi_{\beta}} \left[(U'(\rvw))^2 \right].
\end{flalign}
Since $U'(\rvw)=-n\lambda \hat{L}_S(\rvw)$, it is negative. The function $x \mapsto x^3$ is concave on $(-\infty, 0)$, while the function $x \mapsto x^2$ is convex. Hence Jensen's inequality gives:
\begin{equation}
    \mathbb{E}_{\rvw \sim \pi_{\beta}} \left[U'(\rvw) \right]^3 \geq \mathbb{E}_{\rvw \sim \pi_{\beta}} \left[(U'(\rvw))^3 \right]
\end{equation}
\begin{equation}
    \mathbb{E}_{\rvw \sim \pi_{\beta}} \left[U'(\rvw) \right]^2 \leq \mathbb{E}_{\rvw \sim \pi_{\beta}} \left[(U'(\rvw))^2 \right].
\end{equation}
Thus,
\begin{flalign}
    A+B=\frac{\partial^2 g(\beta)}{\partial \beta^2} & \geq
    3 \left[-\mathbb{E}_{\rvw \sim \pi_{\beta}} \left[U'(\rvw) \right]^3   + \mathbb{E}_{\rvw \sim \pi_{\beta}} \left[U'(\rvw) \right] \mathbb{E}_{\rvw \sim \pi_{\beta}} \left[U'(\rvw) \right]^2 \right]\\
    &= 3 \left[- \mathbb{E}_{\rvw \sim \pi_{\beta}} \left[U'(\rvw) \right]^3  + \mathbb{E}_{\rvw \sim \pi_{\beta}} \left[U'(\rvw) \right]^3 \right]=0,
\end{flalign}
thus $g(\beta)$ is convex.
\end{proof}
\subsubsection{High-probability bounds}
\label{high-prob_bounds_section}
\paragraph{The kl inverse bound}
To ensure approximate independence, the bounds were calculated on a thinned version of the Gibbs samples, which ensured that the ESS of the thinned sample is close to the remaining sample size. This resulted in retaining 1000-3000 samples out of 19800, depending on the dataset. For comparison, MFVI was also evaluated using the same amount of (exact) Gaussian samples.  
\paragraph{The asymptotic bound}
The Markov chain Central Limit Theorem (MCCLT) provides a confidence interval, but the conditions of this theorem are very strong. Fortunately, \cite{clt_stuff} provide a similar, but weaker confidence interval that does not require the MCCLT. The assumptions on our estimator $\frac{1}{m}\sum_{i=1}^m \widetilde{L}^{\mathrm{ce}}_S(\rvw_i)$, $\rvw_i \sim Q$ are as follows.
\begin{enumerate}
    \item $O\left(\frac{1}{n}\right)$ variance of the estimator
    \item Bias of order smaller than $O\left(\frac{1}{\sqrt{n}}\right)$
    \item An estimator $\hat{\sigma}_m$ for the \textit{asymptotic variance} $\sigma:=\lim_{m \to \infty}{m \mathrm{Var}_{\rvw \sim Q}\left[\widetilde{L}^{\mathrm{ce}}_S(\rvw)\right]}$ that converges in probability.
\end{enumerate}

\noindent Assumptions 1. and 2. are typical in MCMC applications \citep{clt_stuff}. For assumption 3, we estimate the asymptotic variance using the fact that the MCMC standard error converges to $\frac{\sigma}{\sqrt{m}}$, following one of the suggestions of \cite{clt_stuff}. The (one-sided) version on the confidence interval in \cite{clt_stuff} has form
\begin{equation}
    I_{m, \epsilon}=\left[0, \frac{1}{m}\sum_{i=1}^m \hat{L}_S(\rvw_i) + m^{-\frac{1}{2}}\hat{\sigma}_m \left( 2\alpha\right)^{-\frac{1}{2}}(1+ \epsilon) \right) \quad \text{with prob. at least } 1-\alpha,
\end{equation}
where $m$ is the number of samples, $\hat{\sigma}_m$ is an estimate of the asymptotic variance, and $\epsilon$ appears in the proof when formalizing Assumption 3. We take $\epsilon=0.01$. Fair comparison to MFVI demands that we use the classical CLT to obtain an asymptotic confidence interval in this case. This is valid as both the transformed cross-entropy and the $0-1$ losses are bounded. The CLT confidence interval has form $I_{\alpha}=\left[0,  \frac{1}{m}\sum_{i=1}^m \hat{L}_S(\rvw_i) + q_{\alpha}\frac{\hat{\sigma}_m}{\sqrt{n}} \right)$, where $q_{\alpha}$ is the appropriate quantile of the standard Gaussian distribution\footnote{Although the population variance is unknown, we decided to use this instead of the Student-t distribution to ensure comparability to $I_{m, \epsilon}$, and because we have a large sample size, $m=20,000$ across the four chains.}.

\paragraph{The naive bound}
Given the difficulties in verifying assumptions for MCMC samples, the reader may wonder if we can give, perhaps looser, estimates for $C(Q^*_{\lambda}, \lambda)$ with milder assumptions. Let $Q$ denote the underlying distribution of our HMC samples, which may not be Gibbs posterior. The following simple proposition can be used to give a (loose) bound for $\mathrm{KL}(Q||P)$ only requiring that $Q$ is not much further from the Gibbs posterior than the mean-field approximation. We expect that $Q$ is much closer to the Gibbs posterior than MFVI, hence this is a very mild assumption.

\begin{prop}
\label{sillyprop}
    Let $P$ be the prior, and let $Q^*_{\lambda}$ denote the corresponding (Gibbs) posterior $Q^*_{\lambda} \propto e^{-n\lambda \widetilde{L}_S^{\mathrm{ce}}(\rvw)}p(\rvw)$. Suppose that we are able to simulate from a distribution $Q$. Let $G$ be another distribution (in our case, a Gaussian) such that $\mathrm{KL}(Q||Q^*_{\lambda}) \leq \mathrm{KL}(G||Q^*_{\lambda}) + n\mathbb{E}_{Q}\left[\widetilde{L}_S^{\mathrm{ce}}(\rvw) \right] $. Then,
\begin{equation}
    \mathrm{KL}(Q||P) \leq n \lambda \mathbb{E}_{G}\left[\widetilde{L}_S^{\mathrm{ce}}(\rvw)\right]  + \mathrm{KL}(G||P).
\end{equation}
\end{prop}
\begin{proof}
    We defer the proof until the end of this section.
\end{proof}
Let us denote our bound objectives as $C(Q, \lambda)$.
We can use this estimate to obtain an upper bound on $C(Q, \lambda)$, which we denote by $C^{\mathrm{UB}}(Q, \lambda)$. For this, we need an estimate for $\widetilde{L}_S^{\mathrm{ce}}(Q)$, for which we use the kl inverse bound (Theorem \ref{caruana_theorem}). It can be shown easily that we have $C^{\mathrm{UB}}(Q, \lambda) \geq C(G,\lambda)$ if we use the $\lambda$ bound and the cross-entropy loss. However, we are most interested in risk certificates in terms of accuracy. For the $0-1$ loss, our upper bound on the $\lambda$ bound becomes

\begin{equation}
    L^{\text{0-1}}(Q) \leq \frac{\hat{L}_S^{\text{0-1}}(Q)}{1-\frac{\lambda}{2}} + \frac{\mathrm{KL}(G||P) + \log (2 \sqrt{n}/\delta)}{n \lambda (1-\frac{\lambda}{2})} + \frac{\widetilde{L}^{\mathrm{ce}}_S(G)}{1-\frac{\lambda}{2}}=:C^{\mathrm{UB}}(Q, \lambda).
\end{equation}
This can result in $C^{\mathrm{UB}}(Q, \lambda) \leq C(G,\lambda)$ if $Q$ has lower $0-1$ loss than $G$. The difference between this and the true $\lambda$ bound on $L^{\text{0-1}}(Q)$ is precisely $\mathrm{KL}(G||Q^*_{\lambda})-\mathrm{KL}(Q||Q^*_{\lambda})+\widetilde{L}^{\mathrm{ce}}_S(Q)$. We will use this bound for a sanity check on our results. If our computed value for $C(Q, \lambda)$ is much higher than this bound, then we likely overestimated the true value, even if our bound is smaller than the one with MFVI. If, on the contrary, our estimate is much lower than this upper bound, that means that either $\mathrm{KL}(G||Q^*_{\lambda})-\mathrm{KL}(Q||Q^*_{\lambda})+\widetilde{L}^{\mathrm{ce}}_S(Q)$ is high or our chains did not even manage to achieve $\mathrm{KL}(Q||Q^*_{\lambda})\leq \mathrm{KL}(G||Q^*_{\lambda})+\widetilde{L}^{\mathrm{ce}}_S(Q)$.

\textit{Proof of Proposition \ref{sillyprop} }
We use Proposition \ref{kl_eq} with $P$ having density $p(\rvw)$, $G$ having density $q(\rvw | z)$ and $Q^*_{\lambda}$ having density $p(\rvw |z)$ and $Z=\mathbb{E}_{P}\left[e^{-n\lambda \widetilde{L}_S^{\mathrm{ce}}(\rvw)} \right]$ being the marginal likelihood. We get
    \begin{equation}
        \mathrm{KL}(G||P) = -n \lambda \mathbb{E}_{G}\left[\widetilde{L}_S^{\mathrm{ce}}(\rvw) \right] - \log(Z) + \mathrm{KL}(G||Q^*_{\lambda}).
\label{furaeq}
    \end{equation}
Reordering this, we obtain an estimate for $-\log(Z)$:
\begin{equation}
        - \log(Z) = \mathrm{KL}(G||P) - \mathrm{KL}(G||Q^*_{\lambda})  +n \lambda \mathbb{E}_{G}\left[\widetilde{L}_S^{\mathrm{ce}}(\rvw) \right].
\end{equation}
This estimate can be used to calculate $\mathrm{KL}(Q||P)$, invoking Proposition \ref{kl_eq} again.
\begin{equation}
    \begin{split}
        \mathrm{KL}(Q||P) &= n \lambda \mathbb{E}_{G}\left[\widetilde{L}_S^{\mathrm{ce}}(\rvw) \right] + \mathrm{KL}(G||P)\\
    &+ \left(  -n\lambda \mathbb{E}_{Q}\left[\widetilde{L}_S^{\mathrm{ce}}(\rvw) \right] +\mathrm{KL}(Q||Q^*_{\lambda}) - \mathrm{KL}(G||Q^*_{\lambda})\right)
    \end{split}
\end{equation}
Since we assumed that $\mathrm{KL}(Q||Q^*_{\lambda}) \leq \mathrm{KL}(G||Q^*_{\lambda}) +   n\lambda \mathbb{E}_{Q}\left[\widetilde{L}_S^{\mathrm{ce}}(\rvw) \right]$, the last term is negative. By ignoring it, we thus obtain an upper bound on $\mathrm{KL}(Q||P)$.

\noindent We note that for $G \equiv P$, Equation \ref{furaeq} becomes Jensen's inequality for the function $-\log$, i.e.,
\begin{equation}
    -\log \mathbb{E}_{G}\left[e^{-n\lambda \widetilde{L}_S^{\mathrm{ce}}(\rvw)} \right] \leq  -\mathbb{E}_{G}\left[\log\left(e^{-n\lambda\widetilde{L}_S^{\mathrm{ce}}(\rvw) }\right)\right].
\end{equation}

\subsection{Implementation details}
\label{implementation_section}
\subsubsection{Transforming the cross-entropy loss}
\label{trans_loss}
PAC-Bayes bounds assume a loss function bounded in $[0, 1]$. However, the cross-entropy loss is unbounded from above. Hence we transform $l^{\mathrm{ce}}(\rvw, y)=\sum_{i=1}^k y_i \log(p_i) $ as follows.
\begin{equation}
\label{pmin}
    \widetilde{l^{\mathrm{ce}}}(\rvw, y) := \frac{-\sum_{i=1}^k y_i \log(\max(p_i, p_{\text{min}}))}{-\log (p_{\text{min}})},
\end{equation}
with some $p_{\text{min}} >0$, which now falls into $[0, 1]$. We will take $p_{\text{min}}=10^{-4}$. The corresponding risk functionals will be denoted by $\widetilde{L}^{\mathrm{ce}}(Q)$ and $\widetilde{L}_S^{\mathrm{ce}}(Q)$.

\subsubsection{Estimating the kl bound}
\label{estim_kl_bound}
After arriving at a posterior distribution $Q$ either by sampling from the Gibbs posterior or MFVI, we wish to compute a risk certificate on the \textit{error} (0-1 loss) $L^{0-1}(Q)$ of the stochastic predictor given by $Q$. The transformed cross-entropy loss is only used for sampling. To compute the risk certificate, we use the kl bound (Equation \ref{pac_kl}) since it is the tightest. To invert the Bernoulli KL, we define
\begin{equation}
    \mathrm{kl}^{-1}(x, b):=\sup \{y \in [x, 1]: \mathrm{kl}(x||y)\leq b \}.
\end{equation}
This can be seen as a proper definition of the inverse Bernoulli KL. Our RC is then
\begin{equation}
    L^{0-1}(Q) \leq \mathrm{kl}^{-1} \left(\hat{L}^{0-1}_S(Q), \frac{\mathrm{KL}(Q||P) + \log (\frac{2 \sqrt{n}}{\delta})}{n} \right).
\end{equation}
\subsubsection{Architectures}
\label{architectures}
We use simple Multi-Layer Perceptron (MLP) architectures in our sampling experiments. Due to the scalability limits of marginal likelihood estimation, the number of parameters are kept small. The architectures along with the corresponding number of parameters are given in Table \ref{architectures_tab}.
\begin{table}[h!]
\caption{Our neural network architectures.}
\centering
\begin{tabular}{@{}lllllc@{}}
\toprule
Dataset              & \multicolumn{4}{l}{Architecture} & \multicolumn{1}{l}{$\#$ params} \\ \midrule
Binary MNIST         & 784     & 20     & -      & 2     & 15742                           \\
$14 \times 14$ MNIST & 196     & 128    & 128    & 10    & 43018                           \\
MNIST                & 784     & 128    & 128    & 10    & 118282                          \\ \bottomrule
\end{tabular}
\label{architectures_tab}
\end{table}
\subsubsection{Hyperparameters}
\label{hyperparam}
\paragraph{Prior hyperparameters} We initialize the prior means randomly from a truncated Gaussian distribution of mean 0 and separate variance for each layer, given by $\frac{1}{\sqrt{n_{\text{in}}}}$, where $n_{\text{in}}$ is the input dimensionality of the layer. The constants of the truncation are $\pm 2$ standard deviations. The prior covariance is set to $0.03 \cdot \mathrm{I}$, matching \cite{tighter}.
We use the cross-entropy loss as our generalized likelihood function, in the transformed version (Equation \ref{pmin}), with $p_{\mathrm{min}}=10^{-4}$. This ensures that our loss is bounded. 

\paragraph{MFVI hyperparameters} As optimizer, we used SGD with learning rate $0.005$ and momentum $0.95$ and train for 150 epochs. In the case of MFVI, the $\lambda$ bound objective can be easily optimized in $\lambda$ too. To match our proposed method, we fix $\lambda=1$ for MFVI too. 

We ran experiments with optimizing for $\lambda$ too, but the best RCs were achieved when $\lambda$ changed very little during the training (at most $0.02$ across our experiments). For these experiments, we followed \cite{tighter} in optimizing $\lambda$ separately using SGD with the same learning rate and momentum used for MFVI. We have run initial experiments with higher learning rates for $\lambda$, but these resulted in looser risk certificates using the kl bound.

We choose our risk certificates to hold with probability $1-\delta$ with $\delta=0.05$. We use $\delta_1=0.025$ to upper bound $\frac{1}{m}\sum_{i=1}^mL_S(\rvw_i)$, $\rvw_i \sim Q$ and then $\delta_2=0.025$ to compute bounds. This ensures a risk certificate that holds with probability $0.95$, via a union bound argument.

\begin{figure}[t]
    \centering
    \includegraphics[width=\textwidth]{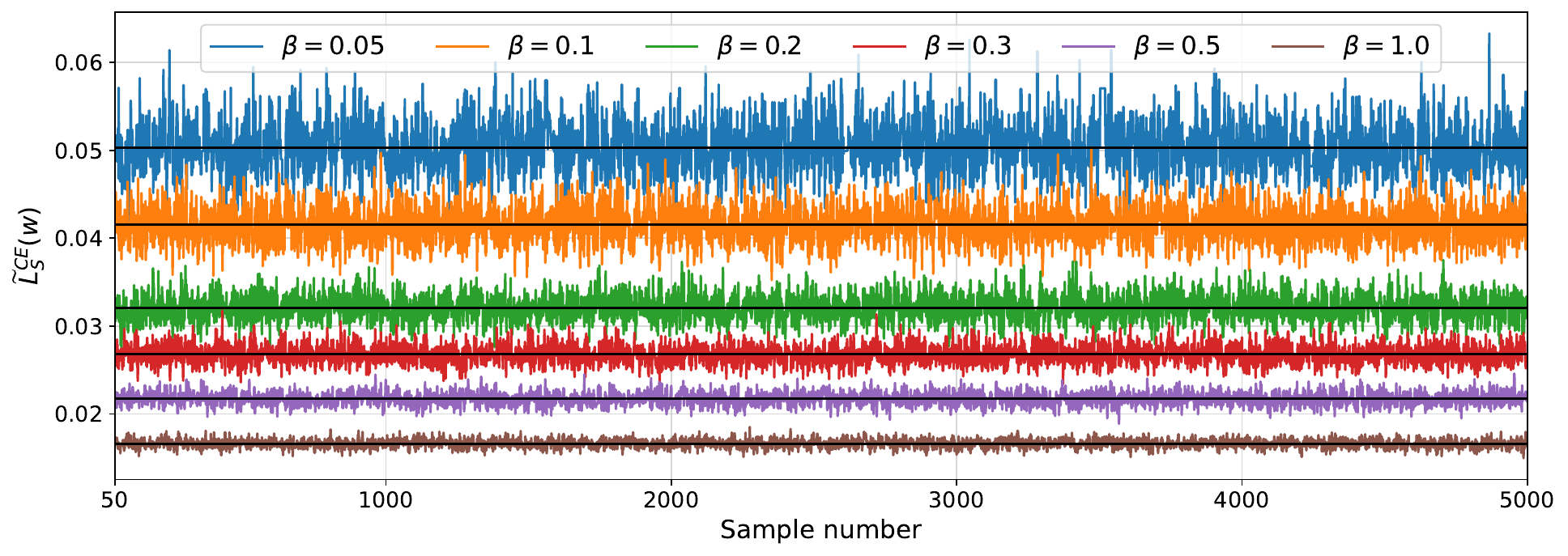}
    \caption[HMC sample traces]{HMC sample traces for a single chain without burn-in (50), for various $\beta$ values. The black lines show the mean of each chain. The dataset is Binary MNIST - Half.}
    \label{tracefig}
\end{figure}

\paragraph{HMC hyperparameters}  We use full-batch HMC to sample from Gibbs posteriors of form $\propto e^{-\beta \widetilde{L}^{\mathrm{ce}}_S(\rvw)}p(\rvw)$ for $\beta \in (0, 1)$ using the JAX HMC implementation of \cite{izmailov}. We run HMC with Metropolis-Hastings correction. We use constant step-sizes for the discretization of Hamiltonian dynamics, which we calibrate individually by testing the values $\{2, 3, 4, 5, 6, 7, 8, 9, 10, 20, 30\}\times 10^{-3}$ for each dataset and $\beta$. The step-sizes were chosen from this set by running HMC for around $100$ sampling iterations for each step size, targeting the ideal acceptance probability for HMC given by $0.65$ \citep{optimal_hmc}. Longer trial runs were not possible due to computational constraints. Lower step sizes were found suitable for higher $\beta$ values.  In every experiment, we use a trajectory length of $1.5$. The authors of \cite{izmailov} recommend a trajectory length of at least $\frac{\pi \sigma_{\mathrm{prior}}}{2}$, which is approximately $0.27$ in our case, but we choose higher than this in order to allow for more leapfrog steps. The chosen step-sizes resulted in leapfrog steps between $50$ and $750$. The difference in leapfrog steps is motivated by the fact that for lower $\beta$ values, the generalized posterior is more similar to the prior and hence is easier to sample from. Each chain was run for $5000$ iterations including a burn-in of $50$ iterations. A low burn-in was chosen since we observed that the chains converge very quickly in function space ($\widetilde{L}^{\mathrm{ce}}_S(\rvw)$). To support this choice, Figure \ref{tracefig} shows burned-in chains for the Binary MNIST - Half dataset for various values of $\beta$. To be able to test the convergence of the chains, we run four chains started from independent seeds for each $\beta$. 

\subsection{Diagnostics}
This section contains the full description and results of our diagnostic analysis on the HMC samples.
\label{further_diagnostics_Section}
\subsubsection{Diagnostic measures}
 Let $(\mathbf{X}^{(t)})_{t=1}^{\infty}$ be a Markov chain on state space $\mathcal{A}$, whose initial distribution equals its stationary distribution, hence $X_1, X_2, ...$ are identically distributed. Consider a square-integrable function $f: \mathcal{A} \to \mathbb{R}$, whose expected value $\mathbb{E}_{P_{X_1}}[f(X_1)]$ we wish to estimate. Note that square integrability impies $\text{Var}[f(X_1)] \leq \infty$. 
\paragraph{Effective Sample Size} First, we define the effective sample size (ESS), which quantifies the loss of information caused by correlation. 
\begin{definition}(Effective Sample Size)
    The Effective sample size of $(\mathbf{X}^{(t)})_{t=1}^N$ is given by
    \begin{equation}
        \mathrm{ESS}[f]=\frac{N}{1 + 2 \sum_{\tau=1}^{\infty} \rho(\tau)},
    \end{equation}
    where $\rho(\tau)=\mathrm{Cor}(f(X_1), f(X_{1+\tau}))$ denotes the autocorrelation at lag $\tau$.
\end{definition}
The ESS is typically lower than the number of samples, indicating the presence of positive correlations. In MCMC, where states are typically positively correlated, an ESS equaling the number of samples signals uncorrelated samples.
\paragraph{MCMC standard error}
\begin{definition}(MCMC standard error)
    The MCMC standard error of $(\mathbf{X}^{(t)})_{t=1}^N$ is
    \begin{equation}
        \mathrm{MCMC}\_\mathrm{SE}[f]=\sqrt{\frac{\mathrm{Var}[f(X_1)]}{\mathrm{ESS}[f]}}.
    \end{equation}
\end{definition}
The MCMC SE measures the concentration of a sample mean around the true mean, in the sense of the Markov chain Central Limit theorem (MCCLT). It requires the above conditions on the Markov chain, namely (i) stationary chain and (ii) square-integrable $f$ and states that, for large $N$, approximately
\begin{equation}
    \frac{1}{N}\sum_{i=1}^N f(X_i) \sim \mathcal{N}(\mathbb{E}_{P_{X_1}}[f(X_1)], \mathrm{MCMC\_SE}[f]).
\end{equation}
In practice, we don't have access to  $\rho(\tau)$ and $\mathrm{Var}[f(X_1)]$, and estimate them from a finite set of samples. 

\paragraph{$\hat{R}$ statistic} Since ESS requires stationarity to be well-defined, we will also use the $\hat{R}$ (known also as potential scale reduction) to check convergence by comparing multiple independent Markov chains \citep{r_hat}. $\hat{R}$ measures how much the variance of the means between multiple chains exceeds that of identically distributed chains. We have that $\hat{R} \geq 1$, where $\hat{R} = 1$ means perfect convergence. For more discussion on $\hat{R}$, please refer to \cite{r_hat}, and for general discussion on MCMC and convergence diagnostics, see \cite{MCMC_practice} and \cite{handbook_MCMC}.

\paragraph{Expected Calibration Error} Following \cite{izmailov}, we supply expected calibration error (ECE) estimates averaged over all estimates for a given $\beta$. ECE measures model calibration by quantifying how well a model's output pseudo-probabilities match the true (observed) probabilities \citep{ece}. 
\label{mcmc_section_diagnostics}

\subsubsection{Diagnostic results}
Figures \ref{bin_half_diag_figure} \ref{bin_full_diag_figure}, \ref{small_half_diag_figure} \ref{small_full_diag_figure}, \ref{half_diag_figure}, \ref{full_diag_figure} show diagnostic results for all datasets, respectively. The target acceptance probability of $0.65$ was not always achieved on average. On Binary MNIST, acceptance probabilities were higher, around $70-80\%$. In these cases, increasing the step-size resulted in unstable behaviour, with an acceptance probability close to $0$. We followed the usual recommendation for these cases, which is to keep the step-size slightly lower than the highest stable value \citep{hmc_conceptual}. On $14 \times 14$ MNIST, acceptance probabilities were roughly as desired.

\begin{figure}[t]
    \centering

\begin{subfigure}
         \centering
    \includegraphics[width=0.95\textwidth]{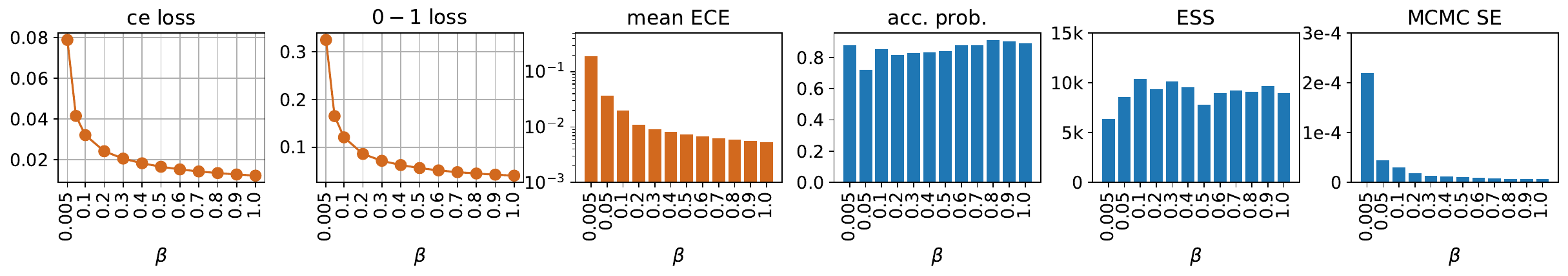}
\end{subfigure}
\begin{subfigure}
         \centering
    \includegraphics[width=0.95\textwidth]{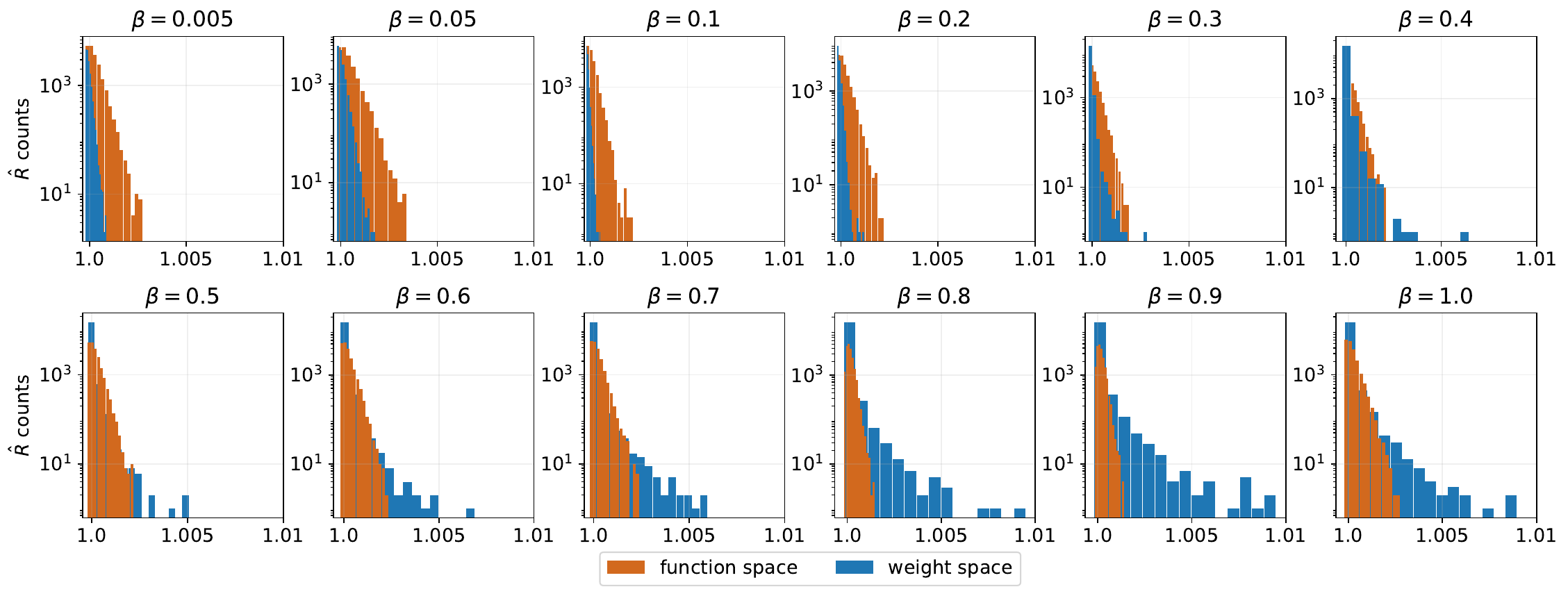}
\end{subfigure}
\caption{HMC diagnostics for Binary MNIST - half, over 4 chains of length 5000, for each $\beta$.}
\label{bin_half_diag_figure}
\end{figure}
\begin{figure}[h!]
    \centering
\begin{subfigure}
         \centering
    \includegraphics[width=0.95\textwidth]{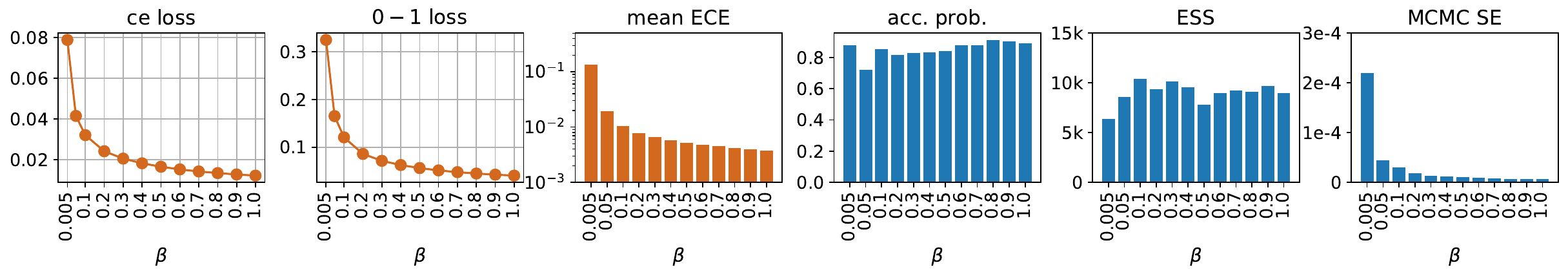}
\end{subfigure}
\begin{subfigure}
         \centering
    \vspace{-0.5em}
    \includegraphics[width=0.95\textwidth]{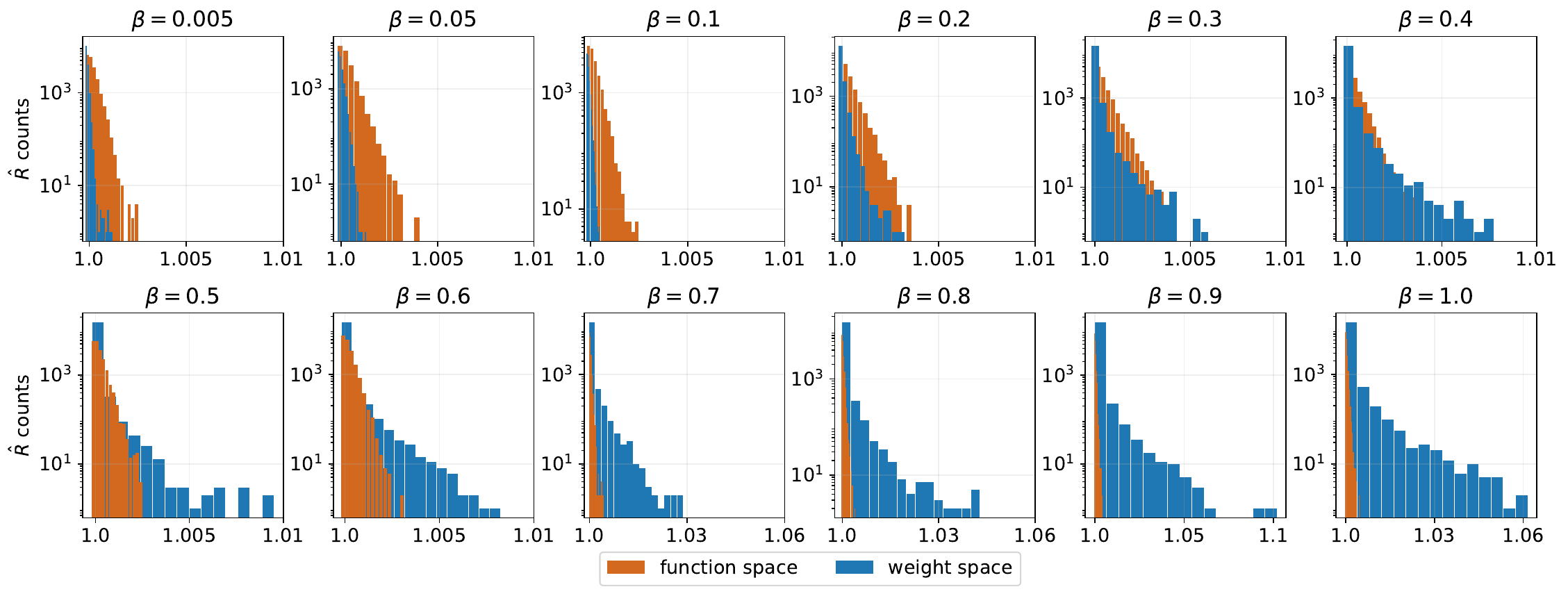}
\end{subfigure}
\caption{HMC diagnostics for Binary MNIST - full, over 4 chains of length 5000, for each $\beta$.}
\label{bin_full_diag_figure}

\end{figure}

  However, the MNIST experiments resulted in some acceptance probabilities below $0.4$. In these cases, the ideal step-size is likely smaller. However, for MNIST, decreasing the step-size would have meant the largest increase in running time. Furthermore, running short chains with smaller step-sizes showed very similar values for the train and test losses. Hence a decision was made to use larger step-sizes to be able to produce a longer chain.

The ESS values are relatively high, almost always retaining at least 5000 out of the 19800 samples (burn-in removed) across the four chains. The MCMC SE values are also low. The sudden decrease as $\beta$ is increased is explained by the fact that the true variance of the Gibbs posterior decreases as $\beta$ increases since the posterior becomes more concentrated around the minima of the loss landscape. We observe low ECE values implying that our samples correspond to well-calibrated models.
All $\hat{R}$ statistics are very close to 1.0, showing that our chains are approximately stationary. This finding motivates our use of the kl inversion bound (Theorem \ref{caruana_theorem}) which required an identically distributed sample.

\begin{figure}[h!]
    \centering
    \vspace{5em}
\begin{subfigure}
         \centering
    \includegraphics[width=0.95\textwidth]{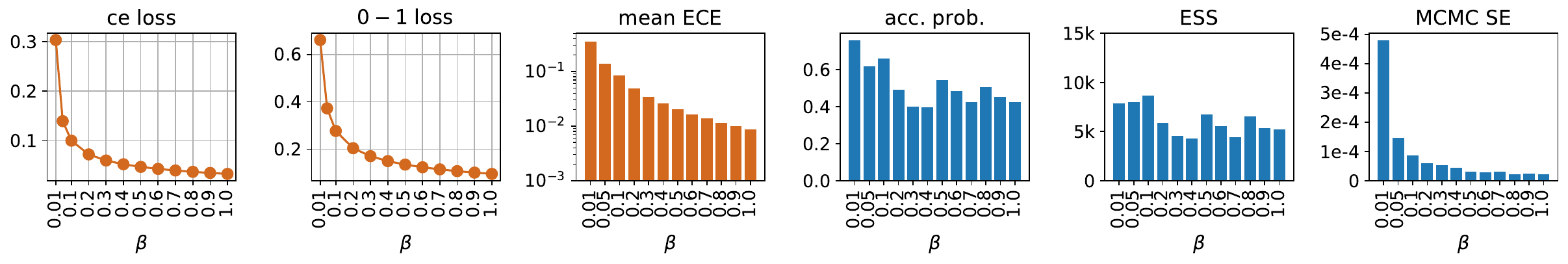}
\vspace{-1em}
\end{subfigure}
\begin{subfigure}
         \centering
    \includegraphics[width=0.95\textwidth]{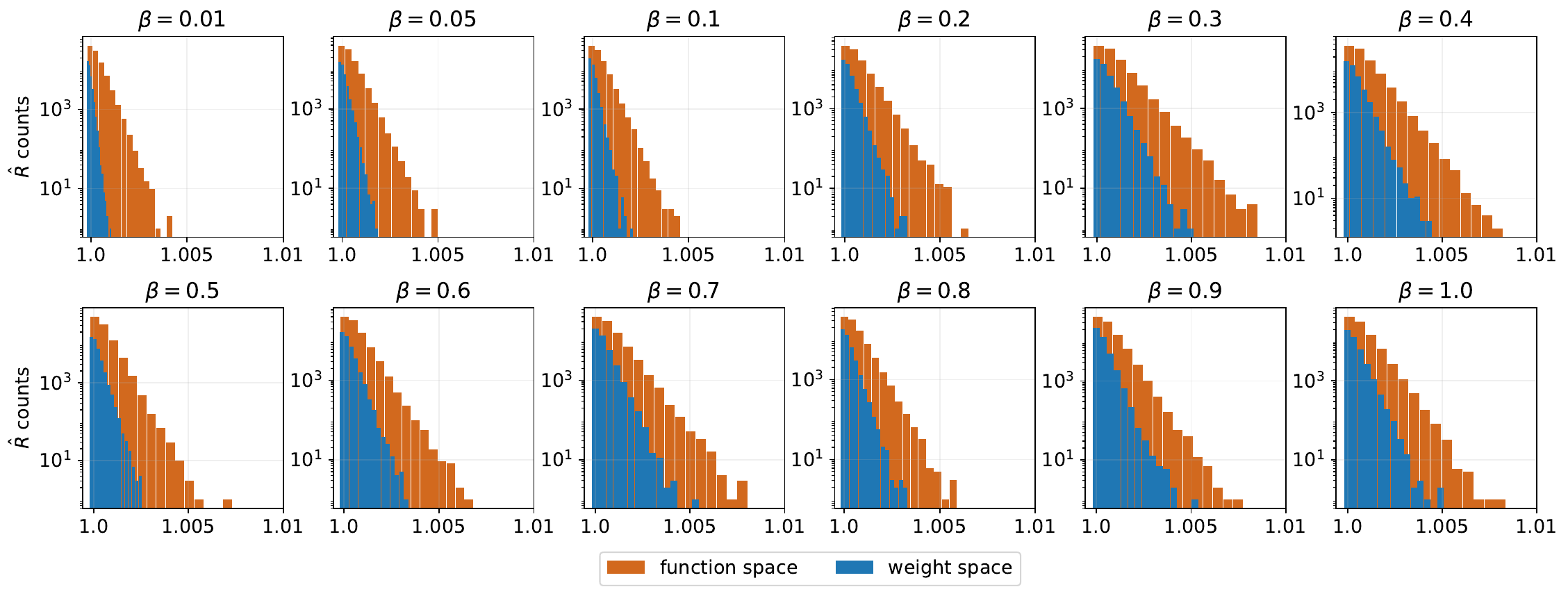}
\end{subfigure}
\caption{HMC diagnostics for 14$\times$14 MNIST - half, over 4 chains of length 5000, for each $\beta$.}
\label{small_half_diag_figure}
\vspace{5em}
\end{figure}

\begin{figure}[h!]
\vspace{2em}
    \centering
\begin{subfigure}
         \centering
    \includegraphics[width=0.95\textwidth]{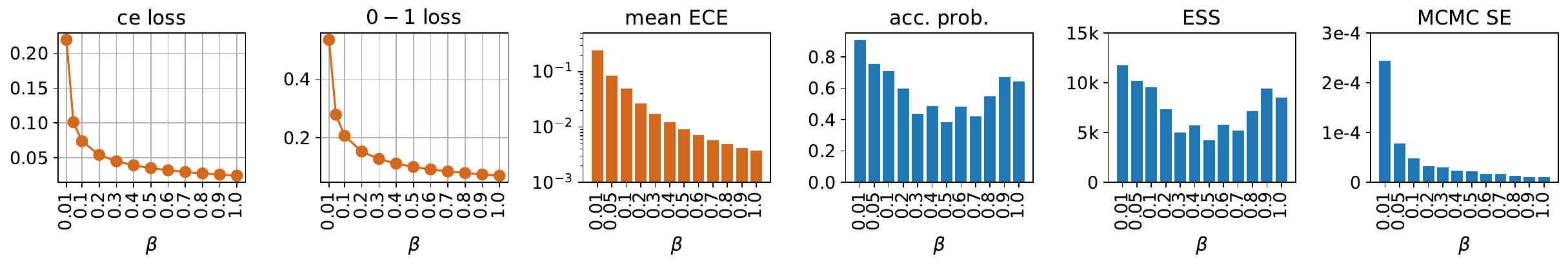}
\end{subfigure}

\begin{subfigure}
         \centering
    \includegraphics[width=0.95\textwidth]{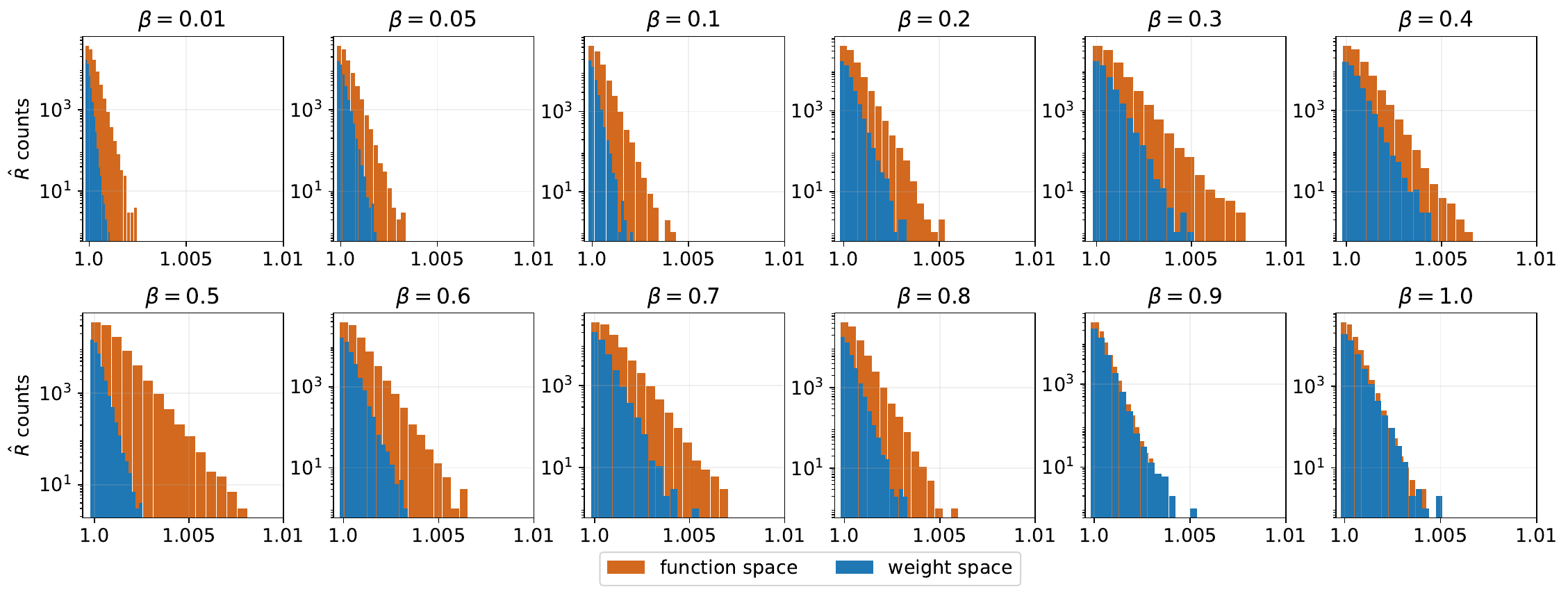}
\end{subfigure}
\caption{HMC diagnostics for $14 \times 14$ MNIST - full, over 4 chains of length 5000, for each  $\beta$.}
\label{small_full_diag_figure}
\end{figure}

\begin{figure}[h!]
    \centering
 
\begin{subfigure}
         \centering
    \includegraphics[width=0.95\textwidth]{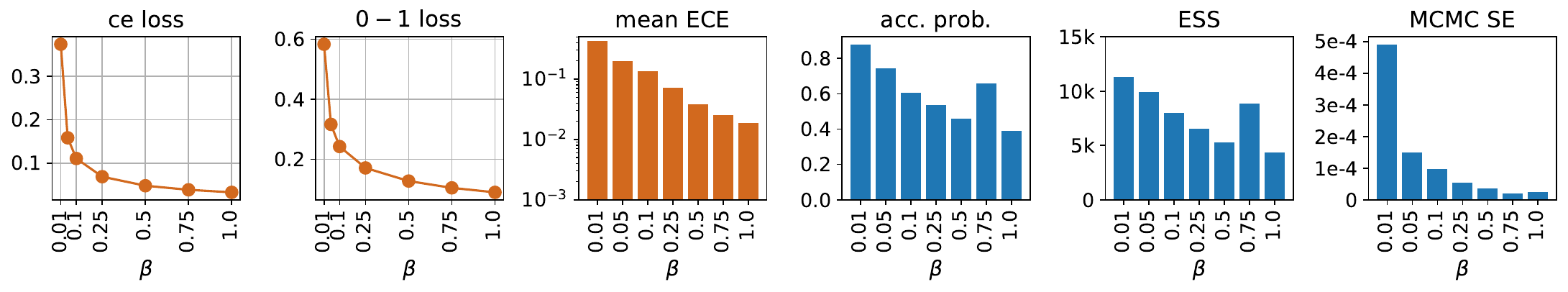}

\end{subfigure}

\begin{subfigure}
         \centering
    \includegraphics[width=0.95\textwidth]{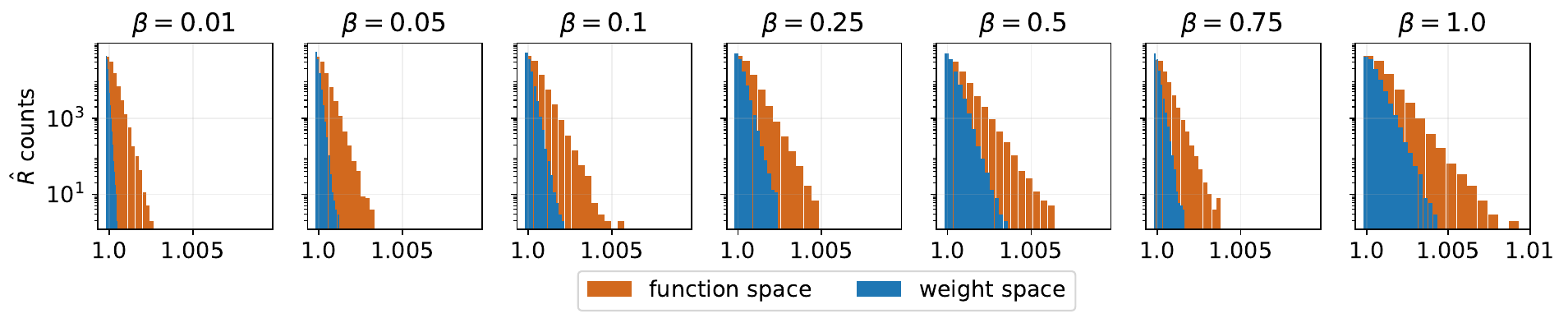}
\end{subfigure}
\caption{HMC diagnostics for MNIST - half, over 4 chains of length 5000, for each value of $\beta$.}
\label{half_diag_figure}
\end{figure}

\begin{figure}[h!]
    \centering
\begin{subfigure}
         \centering
    \includegraphics[width=0.95\textwidth]{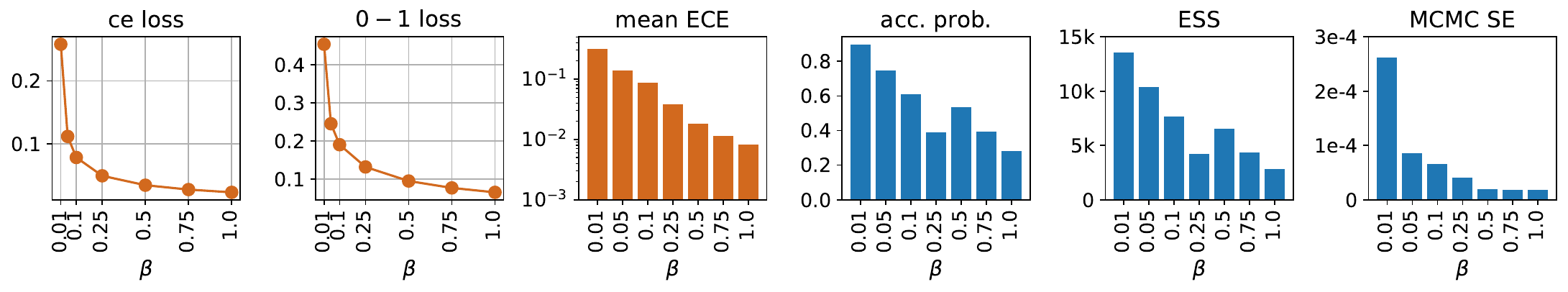}

\end{subfigure}

\begin{subfigure}
         \centering
    \includegraphics[width=0.95\textwidth]{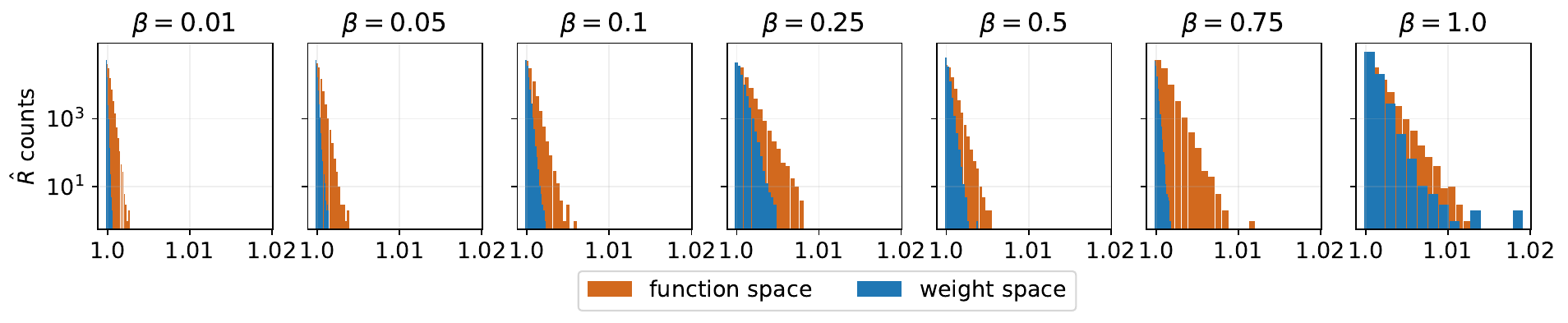}
\end{subfigure}
\caption{HMC diagnostics for MNIST - full, over 4 chains of length 5000, for each value of $\beta$.}
\label{full_diag_figure}
\end{figure}

\newpage
\subsection{Further results and analysis}
\label{further_results_Section}

Table \ref{ce_rc_table} shows RC estimation results for the (bounded) cross-entropy loss. We first observe that there are no bound violations, i.e. the test loss/error is always smaller than the corresponding RC. Further, the RCs for MFVI and the Gibbs posterior have the same magnitude. This suggests that there are no apparent problems with our estimates. Interestingly, the computed RCs for $\widetilde{L}^{\mathrm{ce}}_S(Q)$ (denoted as $\widetilde{\mathrm{ce}}$ in the tables) are almost always better for MFVI than for the approximate Gibbs posterior samples, despite the Gibbs posterior being the minimizer of the $\lambda$ bound. However, the RC-s on the $0-1$ losses (Table \ref{rc_table} above) show considerable improvement with respect to MFVI, especially for Binary MNIST. This discrepancy in cross-entropy and $0-1$ losses may indicate that we have considerably overestimated the KL divergence term. Since our KL divergences are small compared to the $0-1$ losses, they affect the $0-1$ RCs less. We are mainly interested in $0-1$ RCs, hence this is not a problem for us.

Out of the three high-probability bounds compared, the asymptotic bound gives the lowest risk certificates, which is explained by the fact that (i) more samples were used, and (ii) that this interval only guarantees a probability of $0.95$ as the number of samples tends to infinity.

The gap between the $0-1$ RCs for the Gibbs posterior and MFVI increases as the amount of train data is increased. MFVI gives similar RCs in both cases while Gibbs RCs improve by $3-7\%$ when the amount of data is increased. This is reasonable since Gibbs posteriors concentrate more around minima of the loss as $n$ is increased. However, our results show that MFVI is not able to utilize the extra data very well.

\begin{table}[h!]
\caption{Training and test metrics and RC estimates using our three bounds, in terms of the cross-entropy loss. \textbf{Bold} numbers indicate the tighter certificate out of the Gibbs and MFVI ones for the same dataset. }
\resizebox{\textwidth}{!}{
\begin{tabular}{llllllllll}
\hline
\multicolumn{3}{l}{Setup}              & \multicolumn{3}{l}{Train/test stats} & \multicolumn{2}{l}{ce RC with kl bound} & \multicolumn{2}{l}{ce RC with $\lambda$ bound} \\ \hline
Method   & Dataset             & Model & Train $\widetilde{\mathrm{ce}}$    & Test $\widetilde{\mathrm{ce}}$    & KL/n      & kl inverse         & asympt             & kl inverse             & asympt                \\ \hline
MFVI     & Binary Half         & 1L    & 0.0256      & 0.0281     & 0.0106    & \textbf{0.0703}    & 0.0581             & 0.0934                 & 0.0756                \\
Gibbs p. & Binary Half         & 1L    & 0.0166      & 0.0166     & 0.0205    & 0.0706             & \textbf{0.0511}    & \textbf{0.0915}        & \textbf{0.0663}       \\
MFVI     & Binary              & 1L    & 0.0270      & 0.0272     & 0.0105    & \textbf{0.0707}    & 0.0580             & 0.0941                 & 0.0756                \\
Gibbs p. & Binary              & 1L    & 0.0122      & 0.0125     & 0.0195    & 0.0598             & \textbf{0.0315}    & \textbf{0.0777}        & \textbf{0.0402}       \\ \hline
MFVI     & $14 \times 14$ Half & 2L    & 0.0481      & 0.0481     & 0.0148    & \textbf{0.1238}    & 0.0947             & \textbf{0.1710}        & 0.1263                \\
Gibbs p. & $14 \times 14$ Half & 2L    & 0.0335      & 0.0361     & 0.0477    & 0.1410             & \textbf{0.0946}    & 0.1888                 & \textbf{0.1237}       \\
MFVI     & $14 \times 14$      & 2L    & 0.0463      & 0.0460     & 0.0140    & 0.1196             & 0.0906             & 0.1631                 & 0.1194                \\
Gibbs p. & $14 \times 14$      & 2L    & 0.0246      & 0.0259     & 0.0381    & \textbf{0.1118}    & \textbf{0.0702}    & \textbf{0.1484}        & \textbf{0.0910}       \\ \hline
MFVI     & MNIST Half          & 2L    & 0.0430      & 0.0437     & 0.0199    & \textbf{0.1176}    & 0.0966             & \textbf{0.1570}        & 0.1263                \\
Gibbs p. & MNIST Half          & 2L    & 0.0324      & 0.0356     & 0.0428    & 0.1347             & \textbf{0.0935}    & 0.1792                 & \textbf{0.1223}       \\
MFVI     & MNIST               & 2L    & 0.0423      & 0.0419     & 0.0196    & \textbf{0.1010}    & 0.0947             & \textbf{0.1317}        & 0.1226                \\
Gibbs p. & MNIST               & 2L    & 0.0233      & 0.0253     & 0.0334    & 0.1065             & \textbf{0.0673}    & 0.1401                 & \textbf{0.0872}       \\ \hline
\end{tabular}}
\label{ce_rc_table}
\end{table}

When increasing model size, the gap between the $0-1$ RCs decreases as we move from Binary MNIST to $14 \times 14$ MNIST and remains roughly the same as we move to MNIST. This supports the hypothesis of variational approximations become more accurate as model depth is increased.




\end{document}